\documentclass[11pt]{article}

\usepackage[OT1]{fontenc}
\usepackage{fullpage}
\usepackage{smile}
\usepackage[breaklinks,colorlinks,
            linkcolor=red,
            anchorcolor=blue,
            citecolor=blue
            ]{hyperref}
\usepackage{setspace} 
\usepackage[utf8]{inputenc} 
\usepackage[T1]{fontenc}    
\usepackage{hyperref}       
\usepackage{url}            
\usepackage{booktabs}       
\usepackage{amsfonts}       
\usepackage{nicefrac}       
\usepackage{microtype}      
\usepackage{bbm}
\usepackage{xcolor}  
\usepackage{nicefrac}

\usepackage[normalem]{ulem}
\usepackage{xfrac}
\newcommand{\half}{\sfrac{1}{2}}

\makeatletter
\let\old@float@makebox\float@makebox
\renewcommand{\float@makebox}[1]{%
  \color@vbox\normalcolor
    \old@float@makebox{#1}%
  \color@endbox}
\makeatother

\DeclareMathOperator*{\lse}{LSE}

\title{Can Large Language Models Play Games? A Case Study of A Self-Play Approach}
\author{
Hongyi Guo\footnote{Northwestern University. \url{hongyiguo2025@u.northwestern.edu}.}\qquad
Zhihan Liu\footnote{Northwestern University. \url{zhihanliu2027@u.northwestern.edu}.}\qquad
Yufeng Zhang\footnote{ByteDance Inc. \url{yufengzhang.ai@gmail.com}.}\qquad
Zhaoran Wang\footnote{Northwestern University. \url{zhaoranwang@gmail.com}.}
}

\begin{document}

\maketitle

\begin{abstract}
    Large Language Models (LLMs) harness extensive data from the Internet, storing a broad spectrum of prior knowledge. While LLMs have proven beneficial as decision-making aids, their reliability is hampered by limitations in reasoning, hallucination phenomenon, and so on. On the other hand, Monte-Carlo Tree Search (MCTS) is a heuristic search algorithm that provides reliable decision-making solutions, achieved through recursive rollouts and self-play. However, the effectiveness of MCTS relies heavily on heuristic pruning and external value functions, particularly in complex decision scenarios. This work introduces an innovative approach that bolsters LLMs with MCTS self-play to efficiently resolve deterministic turn-based zero-sum games (DTZG), such as chess and go, without the need for additional training. Specifically, we utilize LLMs as both action pruners and proxies for value functions without the need for additional training. We theoretically prove that the suboptimality of the estimated value in our proposed method scales with $\tilde\cO\Bigl(\frac{|\tilde \cA|}{\sqrt{N}} + \epsilon_\mathrm{pruner} + \epsilon_\mathrm{critic}\Bigr)$, where \(N\) is the number of simulations, $|\tilde \cA|$ is the cardinality of the pruned action space by LLM, and $\epsilon_\mathrm{pruner}$ and \(\epsilon_\mathrm{critic}\) quantify the errors incurred by adopting LLMs as action space pruner and value function proxy, respectively. Our experiments in chess and go demonstrate the capability of our method to address challenges beyond the scope of MCTS and improve the performance of the directly application of LLMs.
\end{abstract}

\section{Introduction}
\label{sec:intro}

Large Language Models (LLMs) like GPT-4 \citep{achiam2023gpt} have revolutionized the way we interact with artificial intelligence by utilizing massive datasets sourced from the internet. These models store a vast repository of knowledge covering a wide array of subjects, making them invaluable tools for aiding users in various decision-making scenarios. Unlike traditional algorithms, LLMs can process and interpret complex data, providing nuanced insights and solutions for decision making.
Nonetheless, the reliability of LLMs is compromised by issues such as limited reasoning capacity \citep{huang2022towards,berglund2023reversal}, tendency to produce incorrect or ``hallucinated'' information \citep{huang2023survey,zhang2023siren}, etc. Consequently, the development of a dependable LLM-based agent remains a significant and ongoing challenge in the field. We consider turn-based zero-sum games as our testbed due to its inherent complexity which surpasses conventional single-LLM-agent testing environments such as \citet{shridhar2020alfworld,yao2023tree}, and also because the evaluation in games is clear and straightforward.

Monte-Carlo Tree Search (MCTS) \citep{kocsis2006bandit,coulom2006efficient} is a pivotal decision-making algorithm commonly used in game theory and artificial intelligence. It is particularly noted for its application in board games like chess and Go. MCTS operates by systematically exploring potential moves in a game tree through recursive rolling out and self-play, employing a blend of deterministic and probabilistic methods. Despite its effectiveness, MCTS faces limitations due to its dependency on heuristic pruning strategies \citep{champandard2014monte,schaeffer1989history} and external value functions \citep{baier2012beam}. These dependencies may restrict MCTS's efficiency, particularly in intricate decision-making contexts.

In this study, we explore the integration of LLMs with MCTS self-play, utilizing LLMs to enhance the MCTS framework in two key ways. Firstly, LLMs serve as action pruners, accelerating the self-play process by reducing the number of possible rollouts. Secondly, they function as proxies for value functions, which becomes crucial in evaluating potential outcomes when the rollouts reach their maximum depth. This hybrid approach combines the advantages of both LLMs and MCTS. By incorporating LLMs, we significantly increase the efficiency of MCTS self-play, effectively reducing the width and depth of the search tree. Moreover, this method bolsters the performance of LLMs by leveraging MCTS’s strategic planning capabilities, which helps in selecting the most effective solution from the options proposed by the LLMs.

Recently, researchers at DeepMind trained a 270M parameter transformer model using supervised learning on a dataset comprising 10 million chess games annotated with a professional chess engine. The resulting Language Learning Model (LLM) achieved grandmaster-level performance without employing any searching techniques \citep{ruoss2024grandmaster}. Our work diverges in an orthogonal direction. We aim to construct a decision-making agent using readily available LLM products such as \texttt{gpt-4} and \texttt{gpt-3.5-turbo-instruct} supplemented with searching methodologies to enhance performance. Our approach is applicable to a wide array of decision-making problems necessitating prior knowledge and seamlessly transfers to other tasks without the need for additional training efforts.

We conduct a detailed theoretical analysis of our algorithm, focusing on the suboptimality of the estimated value.
The suboptimality is decomposed into two main components: the discrepancy between our value function and the optimal value function within the context of the pruned action space, and the gap between the optimal value function with respect to the pruned action space and the full action space.
We prove that the suboptimality of the estimated value in our proposed method scales with $\tilde\cO\Bigl(\frac{|\tilde \cA|}{\sqrt{N}} + \epsilon_\mathrm{pruner} + \epsilon_\mathrm{critic}\Bigr)$, where \(N\) is the number of simulations, $|\tilde \cA|$ is the cardinality of the pruned action space by LLMs, and $\epsilon_\mathrm{pruner}$ and \(\epsilon_\mathrm{critic}\) quantify the errors incurred by adopting LLMs as action space pruner and value function proxy, respectively.
This formula highlights the impact of the number of simulations and using LLMs as pruner and critic on the precision of our value estimation.

Our experiments on this proposed method produce encouraging outcomes. We conduct experiments across three scenarios: 
{\em (1) Chess puzzle}: figuring out a mating sequence of varying lengths; 
{\em (2) MiniGo}: playing Go game on a reduced \(5 \times 5\) board;
{\em (3) Chess}: playing as the white player with the initiative in a standard chess game.
In all three settings, our method outperforms the standalone use of LLMs and traditional MCTS self-play, showcasing a superior capability to tackle these challenges. These results indicate that the integration of LLMs with MCTS could mark a considerable advancement in artificial intelligence research, especially in fields that require strategic decision-making and game theory insights.

In summary, our contributions are threefold:
\begin{enumerate}
    \item We provide a novel approach that synergizes LLM and MCTS self-play in turn-based zero-sum games. In this methodology, LLMs are utilized both as action pruners and value function proxies, enhancing the efficiency and effectiveness of both LLMs and MCTS.
    \item We give a theoretical analysis of our algorithm and provide a sublinear suboptimality rate guarantee up to errors incurred by pruning the action space via LLMs and using LLMs as critic.
    \item We validate our approach through experiments in three different contexts: chess puzzles, MiniGo, and standard chess games. These experiments reveal that our method surpasses the performance of both standalone LLMs and conventional MCTS in solving these problems, indicating its superior problem-solving capabilities.
\end{enumerate}

\section{Related Work}

\paragraph{LLM Agent}
Recently, through the acquisition of vast amounts of web knowledge, large language models (LLMs) have demonstrated remarkable potential in achieving human-level intelligence. This has sparked an upsurge in studies investigating LLM-based autonomous agents. Recent work like ToT \citep{yao2023tree}, RAP \citep{hao2023reasoning}, and RAFA \citep{liu2023reason} aim to augment the reasoning capabilities of LLMs by utilizing tree-search algorithms to guide multi-step reasoning. TS-LLM \citep{feng2023alphazero} illustrates how tree-search with a learned value function can guide LLMs' decoding ability.

The LLM Agents for games are of our particular interest. 
\citet{akata2023playing,mao2023alympics} evaluate how LLMs solve different game theory problems.
ChessGPT \citep{feng2023chessgpt} collects large-scale game and language dataset related to chess. Leveraging the dataset, they develop a gpt models specifically for chess, integrating policy learning and language modeling. 
Recently, DeepMind trains a 270M parameter transformer model with supervised learning on a dataset of 10 million chess games annotated by Stockfish 16 engine. The trained LLM achieves grand-master level performance \citep{ruoss2024grandmaster}.

\paragraph{Monte-Carlo Tree Search (MCTS)}
Monte Carlo Tree Search (MCTS) is a decision-making algorithm that consists in searching combinatorial spaces represented by trees. MCTS has been originally proposed in the work by \citet{kocsis2006bandit} and \citet{coulom2006efficient}, as an algorithm for making computer players in Go. 
We survey the works related to action reduction and UCT (Upper Confidence Bounds for Tree) alternatives.

In the expansion phase, the MCTS algorithm adds new nodes into the tree for states resulting after performing an action in the game. For many problems, the number of possible actions can be too high. To address this problem, heuristic move pruning strategy can be applied to reduce the action space, such as alpha-beta pruning \citep{schaeffer1989history} and its variants \citep{pearl1980scout,kishimoto2002transposition}.
Another basic techniques of heuristic action reduction is Beam Search \citep{lowerre1976harpy}. It determines the most promising states using an estimated heuristic. These states are then considered in a further expansion
phase, whereas the other ones get permanently pruned. BMCTS \citep{baier2012beam} combines Beam Search with MCTS. Other variants can be found in \citet{pepels2012enhancements,soemers2016enhancements,zhou2018hybrid}.
Game specified heuristic pruning is also common in the literature for Lords of War \citep{sephton2014heuristic}, Starcraft \citep{churchill2013portfolio,justesen2014script}, etc.

Following the insights from stochastic multi-arm bandit (MAB) literature \citep{agrawal1995sample,auer2002finite}, the Upper Confidence Bound for Trees (UCT) in prior works utilizes logarithmic bonus term for balancing exploration and exploitation within the tree-based search. In effect, such an approach assumes that the regret of the underlying recursively dependent non-stationary MABs concentrates around their mean exponentially in the
number of steps, which is unlikely to hold as pointed out in \citep{audibert2009exploration}, even for stationary MABs. This gap is filled by \citep{shah2020non} which stablishes that the MCTS with appropriate polynomial rather than logarithmic bonus term in UCB provides an approximate value function for a given state with enough simulations. This coincides with the empirically successful AlphaGo \citep{silver2017mastering} that also utilizes a polynomial form of UCB.

\section{Preliminary}
\subsection{Deterministic Turn-based Zero-Sum Two-Player Game (DTZG).}
\label{sec:DTZG}
We consider a deterministic turn-based zero-sum two-player game (DTZG), denoted as a tuple \((\cS, \cA, R, \gamma)\), a tuple consist of state space \(\cS\), action space \(\cA\), reward function \(R\), and a discount factor \(\gamma\). We assume the action space is finite and consider the reward function to be stochastic for generality.
Two players, a max player and a min player take turns to act. Without loss of generality, we assume the player who takes the initiative is the max player. We use \(a\) and \(b\) to denote the action of the max player and the min player, and denote their policy as \(\mu: \cS \to \cP(\cA)\) and \(\nu: \cS \to \cP(\cA)\), respectively, where \(\cP(\cX)\) denote the set of distributions on any set \(\cX\).
We define the state as the concatenation of the action sequence by both players from the beginning of the game such that
\begin{align}
    \label{eq:state}
    s_h = a_0 \circ b_0 \circ a_1 \circ b_1 \circ \dots \circ a_{h-1} \circ b_{h-1}.
\end{align}
Suppose the current state is \(s\) and it's the max player's turn. The max player takes an action \(a \sim \mu(s)\) and receives a reward \(R(s, a)\). By our notation in \eqref{eq:state}, the next state is the concatenation of the current state and the action made by the max player, i.e., \(s' = s \circ a\).
The min player then takes an action \(b \sim \nu(s')\) and receives a reward \(R(s', b)\). 
We denote by
\begin{align}
    \label{eq:reward}
    r(s, a, b) = R(s, a) + R(s \circ a, b)
\end{align}
the total reward of the max player and the min player after they both take an action.
Then, the goal of the max (min) player is to maximize (minimize) the following accumulated reward,
\begin{align}
    V^{\mu,\nu}(s) = \EE_{\mu,\nu} \Biggl[\sum_{h=0}^\infty \gamma^{h} r(s_h,a_h,b_h)\Bigggiven s_0 = s\Biggr],\label{eq:V1}
\end{align}
where $\gamma\in[0,1)$ is the discount factor.

The Nash equilibrium policy $(\mu^\star,\nu^\star)$ in the DTZG satisfies $V^{\mu^\star,\nu^\star}(s)=\max_{\mu}\min_{\nu} V^{\mu,\nu}(s)=\min_{\mu}\max_{\nu} V^{\mu,\nu}(s)$.
Since $\min\max$ is a contraction operator and $\gamma\in [0,1)$, we know there exists an optimal value function $V^\star$ for the following minimax optimal Bellman equation
\begin{align}
    V^\star(s)= \max_{a\in\cA}\min_{b\in\cA} \Bigl(\mathbb{E}\bigl[r(s,a,b)\bigr]+\gamma \cdot V^\star(s\circ a\circ b)\Bigr).\label{eq:V_star}
\end{align}
We introduce the concept of {\em half-step} to handle the status where the max player has made a move but the min player has not. For any step \(h \in \NN\), we denote
\begin{align*}
    s_{h + \half} = s_h \circ a_h, \quad
    a_{h + \half} = b_h.
\end{align*}
With the concept of half-step, we define $V^{\star}_{\half}$ by
\begin{align}
    \label{eq:V_half_star}
    V_{\half}^\star(s) &= \min_{b \in \cA} \Bigl(\EE\bigl[R(s, b)\bigr] + \gamma\cdot V^\star(s \circ b)\Bigr).
\end{align}
Combining \eqref{eq:reward}, \eqref{eq:V_star} and \eqref{eq:V_half_star}, we further define
\begin{align}
    \label{eq:V_half_star2}
    V_0^\star(s) := V^\star(s) = \max_{a \in \cA} \Bigl(\EE \bigl[R(s, a)\bigr]+ V^\star_{\half}(s \circ a)\Bigr).
\end{align}
Hence, by \eqref{eq:V_half_star} and \eqref{eq:V_half_star2}, we show that any DTZG can be reduced to a composition of two single-agent MDPs, which enables us to design the corresponding Monte Carlo Tree Search (MCTS) algorithms for DTZG.

\subsection{Monte Carlo Tree Search (MCTS)}
\label{sec:MCTS}
Monte Carlo Tree Search (MCTS) is a heuristic search algorithm for decision processes.
The aforementioned DTZG forms a game tree. MCTS expands the search tree based on random sampling in the game tree. The application of MCTS in games is based on self-play and rollouts. In each rollout, the game is played out to the end by selecting moves at random. The final game result of each rollout is then used to weight the nodes in the game tree so that better nodes are more likely to be chosen in future rollouts. Each round of MCTS consists of four steps:

{\em (1) Selection}: Start from the root node and select successive child nodes until a leaf node is reached.

{\em (2) Expansion}: Unless the game is done, create child nodes and choose from one of them. Child nodes are valid moves from the game position.

{\em (3) Rollout}: Complete one random rollout from the child node by choosing random moves until the game is decided.

{\em (4) Back-propagation}: Use the result of the rollout to update information in the nodes on the path from \(s_0\) to \(s_h\).

The main difficulty in selecting child nodes is maintaining a balance between the exploitation of nodes with high value and the exploration of moves with few simulations. Following the insights from stochastic multi-arm bandit (MAB) literature, the Upper Confidence Bound for Trees (UCT) in prior works adds a logarithmic bonus term such as \(c \sqrt{\frac{\log N(s)}{N(s, a)}}\) to the value term and selects the node with the highest value to expand. 
However, it's established in \citet{shah2020non} that the MCTS with appropriate polynomial rather than logarithmic bonus term in UCB provides an approximate value function for a given state with enough simulations. This coincides with the empirically successful AlphaGo \citep{silver2017mastering} that also utilizes a polynomial form of UCB.

\section{Algorithm}
\label{sec:algo}

\begin{algorithm}[htbp]
\caption{LLM Self-Play with MCTS}
\begin{algorithmic}[1]
\label{alg:mcts}
\REQUIRE
The number of MCTS simulations \(N\), 
the root node \(s_0\), 
search depth \(H\), 
the UCB bonus function \(B: [H] \times \cS \times \cA \to \RR\), 
the pruned action space \(\tilde\cA: \cS \to \cA\),
and the value function proxy \(\hat V: \cS \to \RR\),
\vskip2pt
\STATE \textbf{Initialize}
\(N(s) = 0\), 
\(\tilde v_h(s) = 0\), 
\(\tilde q_h(s, a) = 0\), 
for any \((s, a) \in \cS \times \cA\) and \(h \in \{0, \half, 1, \dots, H\}\).
\FOR{simulation \(n \leftarrow 1, \dots, N\)}
    \STATE 
    \(N(s_0) \leftarrow N(s_0) + 1\)
    \hfill{\color{blue}\texttt{\em \textbf{/* Rollout */}}}
    \FOR{\(h \leftarrow 0, 1, \dots, H - 1\)}
        \STATE \label{line:max_act}
        \(a_h \leftarrow \argmax_{a \in \tilde\cA({s_h})} \frac{\tilde q_h(s_h, a)}{N(s_h \circ a)} + B_h(s_h, a)\)
        \hfill{\color{blue}\texttt{\em /* Max Player */}}
        \STATE 
        \(s_{h+\half} \leftarrow s_h \circ a_h\)
        \STATE
        \(r_h \leftarrow R(s_h, a_h)\)
        \STATE
        \(N(s_h \circ a_h) \leftarrow N(s_h \circ a_h) + 1\)
        \STATE 
        \(b_h \leftarrow \argmin_{b \in \tilde\cA({s_h \circ a_h})} \frac{\tilde q_{h+\half}(s_h \circ a_h, b))}{N(s_h \circ a_h \circ b)} - B_{h+\half}(s_h \circ a_h, b)\)
        \hfill{\em \color{blue}\texttt{/* Min Player */}}
        \STATE 
        \(s_{h+1} \leftarrow s_h \circ a_h \circ b_h\)
        \STATE
        \(r_{h+\half} \leftarrow R(s_h \circ a_h, b_h)\)
        \STATE
        \(N(s_h \circ a_h \circ b_h) \leftarrow N(s_h \circ a_h \circ b_h) + 1\)
    \ENDFOR
    \STATE \(\tilde v \leftarrow \hat V(s_H)\)\hfill{\color{blue}\texttt{\em \textbf{/* Backward Update */}}}
    \FOR{\(h \leftarrow H-1, H-2, \dots, 0\)}
        \STATE 
        \(\tilde q_{h+\half}(s_h \circ a_h, b_h) \leftarrow r_{h+\half} + \tilde v_{h+1}(s_h \circ a_h \circ b_h)\)
        \hfill{\color{blue}\texttt{\em /* Min Player */}}
        \STATE 
        \(\tilde v \leftarrow r_{h+\half} + \gamma \tilde v\)
        \STATE 
        \(\tilde v_{h+\half}(s_h \circ a_h) \leftarrow \tilde v_{h+\half}(s_h \circ a_h) + \tilde v\)
        \STATE 
        \(\tilde q_h(s_h, a_h) \leftarrow r_h + \tilde v_{h+\half}(s_h \circ a_h)\)
        \hfill{\color{blue}\texttt{\em /* Max Player */}}
        \STATE 
        \(\tilde v \leftarrow r_h + \tilde v\)
        \STATE 
        \(\tilde v_h(s_h) \leftarrow \tilde v_h(s_h) + \tilde v\)
    \ENDFOR
\ENDFOR
\ENSURE \(\tilde V(s_0) = \frac{\tilde v_0(s_0)}{N}\)
\end{algorithmic}
\end{algorithm}

Starting from the initial state \(s_0\), our algorithm evaluates the value of \(s_0\) by executing \(N\) simulations using MCTS with self-play, wherein the LLM functions both as an action pruner and a proxy for the value function. This approach is detailed in Algorithm \ref{alg:mcts}. The algorithm unfolds in two primary stages: the rollout phase, where simulations are conducted to explore possible outcomes, and the backward update phase, where the rewards gained are retroactively applied to update the values of preceding nodes in the search tree.

\paragraph{Rollout.}
The rollout phase (line 3-13) in our algorithm combines the {\em selection}, {\em expansion} and {\em rollout} steps that form the core of the traditional Monte Carlo Tree Search (MCTS) algorithm introduced in Section \ref{sec:MCTS}. 
In this phase, we operate under the assumption that the action space has been pruned by the LLM, resulting in a significantly reduced action set \(\tilde\cA\), compared to the original action space \(\cA\).
During any given step up to the search depth limit, the max player maximizing the outcome chooses the action that yields the highest UCB score. This score is calculated as the sum of the empirical mean of accumulated rewards and a polynomial UCB bonus similar to that in \citet{shah2020non} given by the following
\begin{align}
    \label{eq:ucb}
    B_h(s, a) = \beta_h \cdot \frac{N(s)^{\eta_1}}{N(s \circ a)^{\eta_2}},
\end{align}
for any \((h, s, a) \in \cH \times \cS \times \cA\). Here, the function \(N(\cdot): \cS \to \NN\) counts the number of times the corresponding node is visited in our simulation, \(\beta_h\), \(\eta_1\), \(\eta_2\) are constant parameters, and we denote by \(\cH = \{0, \half, 1, \cdots, H-\half\}\) the set of all step and half-step leading up to the leaf nodes.
The rollout phase concludes once the simulation reaches the maximum depth \(H\), at which point we proceed to update the value estimates based on the outcomes of the simulation.

\paragraph{Backward update.}
We conduct backward update (line 14-22) starting from the leaf node \(s_H\). We use the input value function \(\hat V\) as a proxy of the optimal value \(V^\star\) to estimate \(V^\star(s_H)\). For any step or half-step \(h \in \cH\), we define \(\tilde q_h(s, a)\) as the sum of accumulated rewards received every time after taking action \(a\) at state \(s\), and \(\tilde v_h(s)\) as the sum of accumulated rewards after visiting state \(s\). We conduct the backward update following the minimax optimal bellman equation defined in \eqref{eq:V_half_star2} and \eqref{eq:V_half_star}. The difference between the update for the min player and the max player is the \(\gamma\) factor, which is only applied for the min player.

Finally, our algorithm outputs the average value collected at the root node \(s_0\).
\section{Theory}
For the sake of theoretical analysis, we assume the full action space \(\cA\) and the pruned action space \(\tilde \cA\) are both independent of the state.
For the simplicity of notation, we write \([N] = \{1, 2, \dots, N\}\) for any \(N \in \NN_+\).

The goal of our theory is to upper bound the suboptimality of our method, which is defined as the gap between our value estimate $\tilde V$ and the optimal value $V^\star$ defined in \eqref{eq:V_star} such that
\begin{align}
    \label{eq:suboptimality}
    \bigl|\tilde V(s_0) - V^\star(s_0)\bigr| = o(1), 
\end{align}
with the increasing number of simulation number $N$. 
It's noteworthy that deriving a policy from the value estimate is straightforward. For any state \(s\), the action \(a\) that maximizes (minimizes) the value of the next state \(\tilde V(s \circ a)\) is the optimal action for the max (min) player with respect to the value function \(\tilde V\).
In order to prove the above statement, we make the following assumption about the quality of the LLM-based value function proxy.
\begin{assumption}
    [Quality of LLMs as critic]
    \label{asp:critic}
    There exists \(\varepsilon_0 > 0\) such that
    \begin{align}
        \label{eq:critic_quality}
        \norm{\hat V - V^\star}_\infty \le \varepsilon_0, 
    \end{align}
\end{assumption}

Since our algorithm is built upon the pruned action space by LLMs, mirroring the definition of $V^\star$ in \eqref{eq:V_star}, we define the optimal value under such pruned action space as 
\begin{align}
    \tilde{V}^\star(s) = \max_{a\in\tilde\cA}\min_{b\in\tilde\cA} \Bigl(\mathbb{E}\bigl[r(s,a,b)\bigr]+\gamma \cdot \tilde{V}^\star(s\circ a\circ b)\Bigr),\label{eq:V_star_LLM}
\end{align}
where \(r(s, a, b)\) is the total reward of the max player and the min player after they both take an action as defined in \eqref{eq:reward} and \(\tilde\cA\) is the pruned action space.
We decompose the suboptimality in \eqref{eq:suboptimality} with the help of \(\tilde{V}^\star\) as follows,
\begin{align}
    \label{eq:decompose}
    &\bigl|\tilde V(s_0) - V^\star(s_0)\bigr|
    = \underbrace{\bigl|\tilde V(s_0) - \tilde V^\star(s_0)\bigr|}_{\epsilon_\dag\text{: estimation error}}
    + \underbrace{\bigl|\tilde V^\star(s_0) - V^\star(s_0)\bigr|}_{\epsilon_\ddag\text{: pruning error}}.
\end{align}

Here, the estimation error is the gap between the value estimated by our MCTS algorithm and the true optimal value function under the same pruned action space. The pruning error is the gap between optimal value function on the original action space and the pruned action space.
To characterize the estimation error, we give the following theorem.
\begin{theorem}
    [Estimation Error]
    \label{thm:main}
    Set \(\eta_1 = \sfrac{1}{4}\) and \(\eta_2 = \sfrac{1}{2}\) in Algorithm \ref{alg:mcts}. For any \(s_0 \in \cS\), there exists constants \(\{\beta_h\}_{h \in \cH}\) such that 
    \begin{align*}
        \epsilon_\dag 
        = \bigl|\tilde V(s_0) - \tilde V^\star(s_0)\bigr| 
        \le \gamma^H \varepsilon_0 + \cO\biggl(\frac{|\tilde \cA|}{\sqrt{N}}\biggr),
    \end{align*}
    where \(\gamma\) is the discount factor, \(H\) is the search depth, \(N\) is the number of simulations in Algorithm \ref{alg:mcts}, \(\tilde\cA\) is the pruned action space, \(\varepsilon_0 = \norm{\hat V - V^\star}_\infty\) is the estimation error of the value function proxy defined in Assumption \eqref{asp:critic}.
\end{theorem}
\begin{proof}
    See Appendix \ref{sec:proof} for a detailed proof.
\end{proof}
The analysis of the estimation error follows a similar idea with the single-agent MCTS non-asymptotic analysis \citep{shah2020non}. We show that, the backward update in Algorithm \ref{alg:mcts} of the min player and the max player mirrors twice of value iteration, and thus is contractive. Then, we use induction to prove 

To quantify the pruning error and give an error bound of order \(\cO(1 / \sqrt{N})\), we utilize the LogSumExp (LSE) operator defined as follows
\begin{align}
    \lse(f, \mathcal{X}, \tau) = \frac{1}{\tau} 
    \cdot\log \biggl(\frac{1}{|\mathcal{X}|}\cdot \sum_{x\in\mathcal{X}} \exp\bigl(\tau \cdot f(x)\bigr)\biggr)\label{eq:lse}
\end{align}
for any finite set $\mathcal{X}$, function $f:\mathcal{X}\mapsto\mathbb{R}$, and $\tau\ge0$.
This function well approximates the maximum item of \(X\), as stated in the following lemma.
\begin{lemma}
    \label{lma:lse}
    For any finite set $\cX$ and \(\tau > 0\), it holds that \label{lem:lse}
    \begin{align*}
        \Bigl|\lse(f,\cX, \tau) - \max_{x\in\cX} f(x)\Bigr| \le \frac{\log |\cX|}{\tau},
    \end{align*}
    where the LogSumExp operator \(\lse\) is defined in \eqref{eq:lse}.
\end{lemma}
\begin{proof}
    See Appendix \ref{sec:proof_lse} for a detailed proof.
\end{proof}
By Lemma \ref{lma:lse}, the LogSumExp operator coincides with the \(\max\) operator when \(\tau \to \infty\).
For any $j\in\{0,\half\}$, we define $Q_j^{\star}(s,a)$ as
\begin{align*}
    Q_j^{\star}(s,a)=\EE \bigl[(-1)^{2j}\cdot R(s, a)\bigr]+ (-\gamma)^{2j}\cdot V_{j}^\star(s \circ a),
\end{align*}
where \(V^\star_0\) and \(V^\star_{\half}\) are defined in \eqref{eq:V_half_star2} and \eqref{eq:V_half_star}, respectively.

\begin{definition}
    [Quality of LLMs as pruner] 
    \label{def:llm_approx}
    For any $\tau\ge 0$, we define
    \begin{align}
        \label{eq:pruner}
        \varepsilon_1(\tau) &= \max_{j\in\{0,1/2\}} \max_{s\in\cS}\Bigl| \lse\bigl(Q^\star_j(s,\cdot),\cA,\tau\bigr) -\lse\bigl(Q^\star_j(s,\cdot),\tilde{\cA},\tau\bigr)\Bigr|,
    \end{align}
    where \(\cA\) and $\tilde{A}$ denote the original and pruned action space, respectively.
\end{definition}

\begin{proposition}
\label{prop:V_star_llm}
It holds for any $\tau>0$ and $s\in\cS$ that
    \begin{align*}
        \epsilon_\ddag 
        &= \bigl|V^\star(s) - \tilde V^\star (s)\bigr|
        \le \frac{2}{1 - \gamma} \cdot \biggl(\varepsilon_1(\tau) +\frac{2\log(|\cA||\tilde\cA|)}{\tau}\biggr),
    \end{align*}
    where $\varepsilon_1(\tau)$ is defined in Definition \ref{def:llm_approx}.
\end{proposition}
\begin{proof}
    See Appendix \ref{pf:prop:V_star_llm} for a detailed proof.
\end{proof}

Combining \eqref{eq:decompose}, Theorem \ref{thm:main} and Proposition \ref{prop:V_star_llm} together, we have the following corollary.
\begin{corollary} Under the same settings in Theorem \ref{thm:main}, it holds that
    \label{cor:main}
    \begin{align*}
        &\bigl|\tilde V(s_0) - V^\star(s_0)\bigr| 
        \le \tilde\cO\biggl(\gamma^H \varepsilon_0 
        + \frac{\varepsilon_1(\sqrt N)}{1 - \gamma} 
        + \frac{|\tilde\cA|}{\sqrt N}\biggr),
    \end{align*}
    where \(\tilde\cO\) hides constant factors and logarithms terms and \(|\tilde\cA|\) is the cardinality of the pruned action space.
\end{corollary}

We denote \(\epsilon_\mathrm{critic} = \gamma^H \varepsilon_0\) and \(\epsilon_\mathrm{pruner} = \frac{\varepsilon_1(\sqrt N)}{1 - \gamma}\).
Corollary \ref{cor:main} reveals that the suboptimality of the value estimation in our Algorithm \ref{alg:mcts} vanishes at a sublinear rate with respect to the number of simulations \(N\) up to errors incurred from LLM as critic and action pruner. 
Thus, with a proper chosen simulation number, our algorithm well approximates the optimal value function. Notice that if the LLM does not prune the action space at all, the last error term in Corollary \ref{cor:main} becomes \(\tilde\cO\Bigl(\frac{|\cA|}{\sqrt N}\Bigr)\), which is much bigger since \(|\cA| \gg |\tilde \cA|\).


\section{Experiments}

\begin{figure*}[htbp]
    \centering
    \includegraphics[width=.9\linewidth]{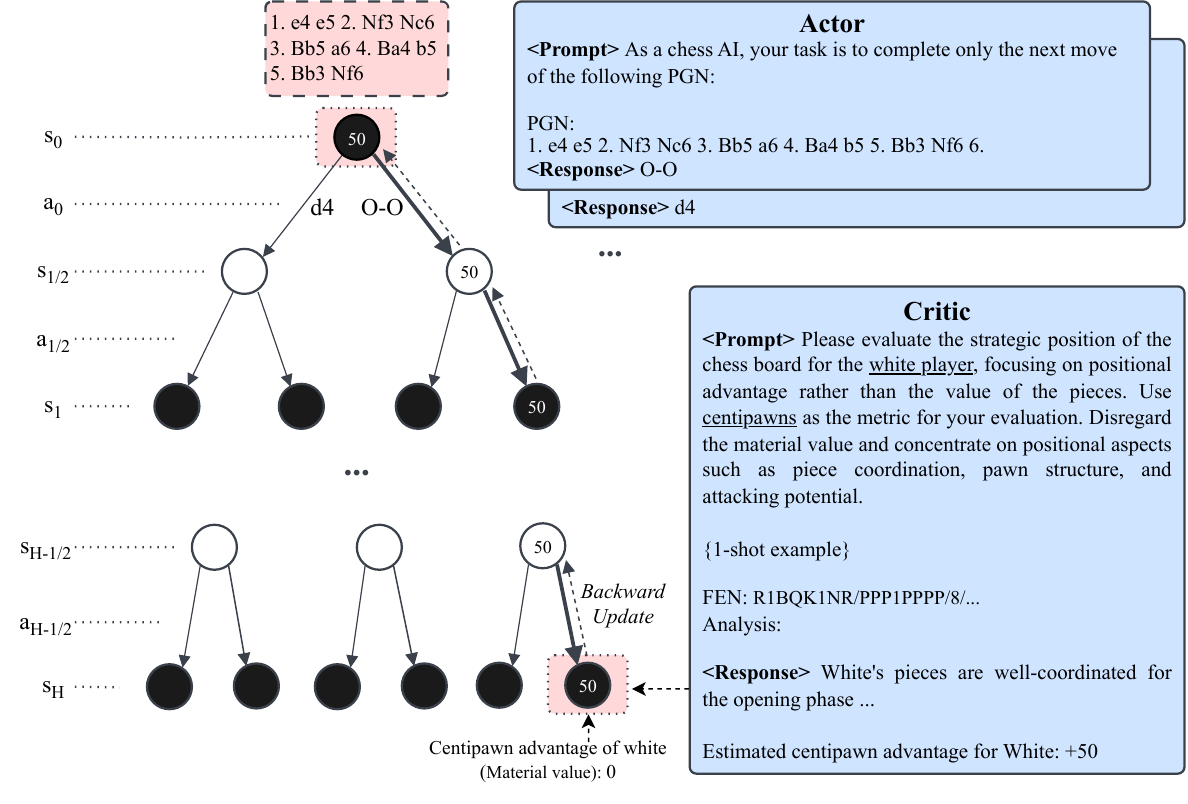}
    \caption{Algorithm illustration for chess. We assume black nodes corresponds to the state where the black player has just made a move, and vice versa for the white nodes.}
    \label{fig:tree}
\end{figure*}

We detail our experiments conducted across three different game scenarios. For each experiment, we configure the LLM as an action pruner, allowing it to suggest what it deems as the optimal action 20 times for each state, with a temperature setting of 0.7.
We set \(\eta_1 = \eta_2 = \half\) and \(\beta_h = 1.25\) for any \(h \in \cH\) for the UCB bonus defined in \eqref{eq:ucb}.
We begin by exploring two initial toy experiments: chess puzzles and MiniGo, where the LLM's role is solely to act as an action pruner. Given the shorter game horizons in these scenarios, we directly use the game's outcome as a proxy for the value function. Specifically, outcomes of winning, losing, and drawing are assigned rewards of 1, -1, and 0, respectively. 
Afterwards, we will present our main experiment conducted on full chess games, where we set a finite search depth and use a hybrid of rule-based and LLM-based value function proxy as illustrated in Figure \ref{fig:tree}. 

\subsection{Chess Puzzle}
Our first experiment setting is chess checkmate puzzles.
We collect chess puzzles from lichess\footnote{\url{https://database.lichess.org/\#puzzles}}. 
The puzzles have different themes describing the topic of the puzzle, such as winning the game with style, or carrying out special moves. We pick the puzzles with themes \textit{mateInN}. The corresponding descriptions are shown in Table \ref{tab:themes}.
\begin{table}[htbp]
    \centering
    \begin{tabular}{cc}
        \toprule
        \textbf{Theme} & \textbf{Description} \\
        \midrule
        \textit{mateIn1} & Deliver checkmate in one move.\\
        \textit{mateIn2} & Deliver checkmate in two moves.\\
        \textit{mateIn3} & Deliver checkmate in three moves.\\
        \textit{mateIn4} & Deliver checkmate in four moves.\\
        \textit{mateIn5} & Figure out a long (\(\geq 5\)) mating sequence.\\
        \bottomrule
    \end{tabular}
    \caption{Themes}
    \label{tab:themes}
\end{table}

Since both the white player and the black player are possible to be the puzzle solver, we name the role corresponding to the puzzle solver as the \textit{agent}, and name the other one as the \textit{opponent}.
We define the ``depth'' of a puzzle as the number of moves needed to deliver checkmate. We denote the state corresponding to the starting point of the puzzle as \(s_0\). For a puzzle with depth \(H\), we would know if the agent has solved the puzzle or not at step \(s_{H-1} \circ a_{H-1} = s_0 \circ a_0 \circ a_{\half} \circ \dots \circ a_{H-1}\). 
Without loss of generality, we assume the agent is the max player and the opponent is the min player. And the agent gets a reward of 1 if he delivers checkmate within required steps, and -1 otherwise.

Since puzzles with depth 1 and 2 are quite simple, we only consider puzzles with depth \(d = 3, 4, \dots, 8\). We construct the evaluation dataset corresponding to each depth by taking one hundred puzzles with the highest rating judged by lichess users.
We implement our Algorithm \ref{alg:mcts} with \texttt{gpt-3.5-turbo-instruct} for its lower cost. 
We compare our method with the following baseline methods:

{\em Vanilla LLM}: The vanilla LLM method uses majority voting to pick action out of it's proposed candidates. We set the temperature to 0.7 and set the number of chat completion choices to 20. And we use majority voting to pick the most popular legal answer from the 20 answers by the LLM. We choose \texttt{gpt-3.5-turbo-instruct} to conduct our experiment to align with our proposed method.


{\em MCTS (50) and MCTS (10,000)}: The vanilla MCTS method with 50 and 10,000 simulations. In this chess puzzle setting, the only difference between this baseline and our proposed algorithm is the number of simulations and the action space. The action space here is the full action space instead of the pruned action space by LLM.

The results are as shown in Table \ref{tab:puzzle}. By combining LLM and MCTS self-play, our method outperforms both LLM and MCTS baselines by a huge margin. Even with only \(50\) simulations, our method outperforms the MCTS baseline with \(10,000\) simulations.

\begin{table*}[t]
    \centering
    \begin{tabular}{ccccc}
        \toprule
        \textbf{Puzzle Depth} & \textbf{Vanilla LLM} & \textbf{MCTS (\(50\))} & 
        \textbf{MCTS (\(10,000\))} & \textbf{LLM + MCTS (\(50\))}\\
        \midrule
        3 & 10\% & 1\% & 33\% & \textbf{74\%}\\
        4 & 15\% & 1\% & 10\% & \textbf{67\%}\\
        5 & 5\% & 0\% & 2\% & \textbf{19\%}\\ 
        6 & 4\% & 0\% & 0\% & \textbf{24\%}\\
        7 & 1\% & 0\% & 0\% & \textbf{37\%}\\
        8 & 0\% & 0\% & 0\% & \textbf{43\%}\\
        \bottomrule
    \end{tabular}
    \caption{Percentage of solved chess puzzles by different methods. The best performance is marked with a bold font.}
    \label{tab:puzzle}
\end{table*}

\subsection{MiniGo}
We use \(5 \times 5\) go board as another toy example. We adopt the ko rule which states that the stones on the board must never repeat a previous position of stones.
We test our method and the baseline methods by letting them play as black with the initiative against a fixed opponent who adopts vanilla MCTS methods with 1,000 simulations and plays as white. Since the game on a \(5 \times 5\) go board is still very short, we set the depth of the MCTS to a large number so that we directly use the game’s outcome as a proxy for the value function. The outcomes of winning, losing, and drawing are assigned rewards of 1, -1, and 0, respectively. Note that the first player always has an advantage in Go. 
We use the territory score as our metric, which is defined by the territory of the black player minus that of the white player. We repeat the battle for 20 times and report the average score.
The results are shown in Table \ref{tab:go}. Our algorithm achieves the highest advantage against the fixed opponent.
\begin{table}[htbp]
    \centering
    \begin{tabular}{ccccc}
        \toprule
        \textbf{LLM} & \textbf{LLM + MCTS (\(50\))} & \textbf{MCTS (\(50\))} & \textbf{MCTS (\(10,000\))} \\
        \midrule
        -8.0 & 11.9 & 1.3 & 9.5 \\
        \bottomrule
    \end{tabular}
    \caption{The score of different methods playing MiniGo (\(5 \times 5\)) against a vanilla MCTS baseline with 1,000 simulations.}
    \label{tab:go}
\end{table}

\subsection{Chess: Full Game}

Our previous experiments on chess puzzles have shown the efficiency of LLM as MCTS pruner. In this part, we evaluate our method against the stockfish engine \footnote{\url{https://stockfishchess.org/}} in full games.
In a full game scheme, the game tree would become deep. Thus, we set a fixed MCTS search depth of \(10\) and use a hybrid value function proxy composed of both a rule based centipawn evaluator and a LLM-based critic.

{\em Rule-based critic:} The rule-based critic we use calculates the piece values of the current board in centipawns. The centipawn is the unit of measure used in chess as measure of the advantage. The value of a pawn is 100 by default. We assign a value for each of piece in Table \ref{tab:centipawn}. Our critic evaluates the current board by subtracting the total value of the current player's pieces by the total value of the opponent's pieces and divide the result by 1,000. For example, if the black player has one more queen than the white player in the current board, the value evaluated by the critic for the black player is 0.9. We denote this rule-based critic as \(V_\mathrm{rule}: \cS \to \RR\).
Specifically, if the game is finished, we set this critic value as 10 for the winner or 0 if the game is a draw.
\begin{table}[htbp]
    \centering
    \begin{tabular}{cccccc}
        \toprule
        \textbf{Piece} & Pawn & Knight & Bishop & Rook & Queen \\
        \midrule
        \textbf{Value} & 100 & 325 & 325 & 500 & 900 \\
        \bottomrule
    \end{tabular}
    \caption{Centipawn value for each kind of piece.}
    \label{tab:centipawn}
\end{table}

{\em LLM-based critic:} The LLM-based critic evaluates the strategic position of the chess board such as piece coordination, pawn structure and attacking potential, focusing on positional advantage rather than the value of the pieces. It also uses centipawn as the unit of metric and we divide the result by 1,000 just like the rule-based critic.
The prompt we use for LLM-based critic are shown in Figure \ref{fig:tree}. We denote the LLM-based critic as \(V_\mathrm{LLM}: \cS \to \RR\).

We combine the rule-based critic and the LLM-based critic by adding them together, i.e. \(\hat V(s) = V_\mathrm{rule}(s) + V_\mathrm{LLM}(s)\). 
In practice, we use \texttt{gpt-4} with a temperature of 0 as the backend of the LLM-based critic, and for consistency, we also use \texttt{gpt-4} as the backend of the actor.
To evaluate the performance, we define a score function that maps ``win'', ``tie'', and ``lose'' to \(1\), \(\half\), and \(0\), respectively. We evaluate our method, the standalone LLMs and conventional MCTS by playing against Stockfish engine with different engine levels.
The results are shown in Table \ref{tab:chess}.
For the MCTS baseline, we run it in two way, one with a search depth of \(512\) so that no critic is needed because the game is simulated to the end, and the other one with a search depth of \(10\) and is equipped with only the rule-based critic \(V_\mathrm{rule}\) same with our proposed method. However, we find that both MCTS baseline approaches failed to achieve a win or draw against the stockfish engine.

\begin{table}[htbp]
    \centering
    \begin{tabular}{cccc}
        \toprule
        \textbf{Level} & \textbf{LLM + MCTS ($50$)} & \textbf{LLM} & \textbf{MCTS ($100,000$)} \\
        \midrule
        1 & \textbf{0.78} & 0.35 & 0\\
        2 & \textbf{0.73} & 0.15 & 0\\
        3 & \textbf{0.47} & 0.05 & 0\\
        \bottomrule
    \end{tabular}
    \caption{The score of different methods playing chess against different levels of stockfish engine.}
    \label{tab:chess}
\end{table}

\section{Conclusion}
In this study, we harness the unique capabilities of Large Language Models (LLMs) as black-box oracles and Monte-Carlo Tree Search (MCTS) as planning oracles to develop a novel self-play algorithm tailored for deterministic turn-based zero-sum games. Our approach positions the LLM to fulfill dual roles: firstly, as an action pruner, narrowing the breadth of the MCTS search tree, and secondly, as a value function proxy, thereby shortening the depth required for the MCTS search tree. This dual functionality enhances the efficiency and effectiveness of both methodologies.

Our research is supported by both theoretical and empirical analyses. Theoretically, we demonstrate that the suboptimality of the estimated value in our method is proportional to \(\cO(\frac{|\tilde\cA|}{\sqrt{N}})\) up to errors incurred by using LLM as critic and action pruner, where \(N\) is the number of simulations and \(\tilde\cA\) is the pruned action space.
Empirically, through tests in chess and Go, we showcase our method's ability to tackle complex problems beyond the reach of traditional MCTS, as well as to outperform the direct application of LLMs, highlighting its potential to advance the field of game theory and artificial intelligence.

\newpage
\bibliography{ref}
\bibliographystyle{ims}

\appendix
\section{Value Estimation Error Analysis}
\label{sec:proof}

Without loss of generality, we assume the root node correspond to the state where the max player is about to make a move. Thus, the leaf nodes correspond to the state where the min player is about to make a move.
For any \(h \in \{0, \half, \dots, H-\half, H\}\), denote by \(n_h\) the number of nodes at level \(h\) and denote by \(s^{(1)}_h, s^{(2)}_h, \dots, s^{(n_h)}_h \in \cS\) the nodes in at level \(h\). 
We denote \(K = |\tilde\cA|\) the cardinality of the pruned action space and represent the pruned action space by \([K]\).
We rewrite the bonus function as 
\begin{align*}
    B_h(s, a) = \beta_h^{\sfrac{1}{\xi_h}} \cdot \frac{N(s)^{\sfrac{\alpha_h}{\xi_h}}}{\sqrt{N(s \circ a)}}.
\end{align*}

For any \((s, a) \in \cS \times [K]\), we denote by
\begin{align*}
    \tilde q_{H-\half}^\star(s, a) = \EE\bigl[R(s, a)\bigr] + \gamma \tilde v_H^\star(s \circ a)
\end{align*}
the estimated expected return for the min player at level \(H - \half\) for taking action \(a\) at state \(s\), where \(\tilde v_H^\star\) is the input value function of Algorithm \ref{alg:mcts} as an estimation of the true value function \(V^\star\).
Since the level \(H - \half\) corresponds to the min player, we define 
\begin{align*}
    \tilde v_{H-\half}^\star(s) = \min_{a \in [K]} \tilde q_{H-\half}^\star(s, a).
\end{align*}
Then, we define recursively 
\begin{align}
    \label{eq:value}
    \tilde q_h^\star(s, a) &= \EE\bigl[R(s, a)\bigr] + \tilde v_{h+\half}^\star(s \circ a),\quad
    \tilde v_h^\star(s) = \max_{a \in [K]} \tilde q_h^\star(s, a)\\
    \label{eq:value_}
    \tilde q^\star_{h + \half}(s, a) &= \EE\bigl[R(s, a)\bigr] + \gamma \tilde v_{h+1}^\star(s \circ a),\quad
    \tilde v^\star_{h + \half}(s) = \min_{a \in [K]} \tilde q_{h + \half}^\star(s, a)
\end{align}
for any \(h = H-1, H-2, \dots, 0\).
Meanwhile, we denote by 
\begin{align*}
    \tilde a_h^\star(s) \in \argmax_{a \in [K]} \tilde q_h^\star(s, a),\quad
    \tilde a_{h+\half}^\star(s) \in \argmin_{a \in [K]} \tilde q_{h+\half}^\star(s, a)
\end{align*}
the optimal action at state \(s\) according to the value function \(q\), and by
\begin{align}
    \label{eq:best2sec}
    \Delta_h(s) = \tilde v_h(s) - \max_{a \in [K], a \ne a_h^\star(s)} \tilde q_h^\star(s, a),\quad
    \Delta_{h + \half}(s) = \tilde v_{h+\half}^\star(s) - \max_{a \in [K], a \ne a_{h+\half}^\star(s)} \tilde q_{h+\half}^\star(s, a)
\end{align}
the gap between the optimal arm and the second optimal arm at state \(s\). Furthermore, we define the empirical value functions of our Algorithm \ref{alg:mcts}. For each level \(h \in \{0, \half, 1, \dots, H - \half\}\), denote by \(\tilde v_h^{(n)}(s)\) the sum of discounted rewards collected at state \(s\) during \(n \in \NN_+\) visits of it.

\subsection{Non-stationary Multi-Arm Bandit}
\label{sec:mab}
Consider non-stationary multi-arm bandit (MAB) problems. Let there be \(K \ge 1\) arms or actions and let \(X_{k,t}\) denote the random reward obtained by playing the arm \(k \in [K]\) for the \(t\)-th time with \(t \ge 1\). We assume that the optimal arm is unique.
We denote by \(\bar X_{k,n} = \frac{1}{n} \sum_{t=1}^n X_{k,t}\) the empirical mean of the reward obtained by playing arm \(k\) for \(n\) times and denote by \(\mu_{k,n} = \EE[\bar X_{k,n}]\) the expectation of the empirical mean. We assume that the optimal arm is unique. We make the following assumptions about the reward \(X_{k,t}\).

\begin{assumption}
    \label{asp:bounded}
    There exists an absolute constant \(R\) such that \(X_{k,t} \in [-R, R]\) for any arm \(k \in [K]\).
\end{assumption}
\begin{assumption}
\label{asp:process}
    The reward sequence \(\{X_{k,t}: t \ge 1\}\) is a non-stationary process such that

    \begin{enumerate}
        \item there exists \(\mu_k\) for any \(k \in [K]\) such that
        \begin{align*}
            \mu_k = \lim_{n \to \infty} \EE\Biggl[\frac{1}{n} \sum_{t=1}^n X_{k,t}\Biggr].
        \end{align*}
        \item there exists \(\beta \in (1, \infty)\) and \(\xi \in (0, \infty)\) such that for any \(z \in [1, \infty)\) and \(n \in \NN_+\),
        \begin{align*}
           \PP\Biggl(\sum_{t=1}^n X_{k,t} - n \mu_k \ge \sqrt{n} z\Biggr) \ge \frac{\beta}{z^\xi},\quad
           \PP\Biggl(\sum_{t=1}^n X_{k,t} - n \mu_k \le -\sqrt{n} z\Biggr) \ge \frac{\beta}{z^\xi}.
        \end{align*}
    \end{enumerate}
\end{assumption}

Assumption \ref{asp:bounded} states that the reward is bounded for all arms. And Assumption \ref{asp:process} establishes the convergence and concentration properties of the process. Those assumptions holds naturally for bounded and deterministic reward setting.

Consider the following variant of UCB algorithm applied to the above non-stationary MAB.
\begin{algorithm}[htbp]
\caption{A variant of Upper Confidence Bound (UCB) algorithm (Max Player)}
\begin{algorithmic}[1]
\label{alg:ucb}
\REQUIRE A non-stationary MAB with \(K\) arms. Parameters \(\alpha\), \(\beta\), \(\xi\).
\STATE Initialize \(T_{k,0} \leftarrow 0\) for any \(k \in [K]\).
\FOR{\(t \leftarrow 0, 1, \dots\)}
    \STATE \(U_{k,s,t} \leftarrow \sum_{\tau=1}^s X_{k,\tau} + B_{t,s}\), where \(B_{t,s} = \frac{\beta^{\sfrac{1}{\xi}} \cdot t^{\sfrac{\alpha}{\xi}}}{\sqrt{s}}\), for any \(s \in [t]\)
    \STATE \(k_t \leftarrow \argmax_{k \in [K]} U_{k, T_{k,t-1}, t-1}\) \quad{\color{blue}\texttt{/* Pull arm \(k_t\) and get reward \(X_{k_t, T_{k,t-1} + 1}\). */}}
    \STATE \(T_{k_t,t} \leftarrow T_{k_t,t-1} + 1\)
    \STATE \(T_{k,t} \leftarrow T_{k,t-1}, \forall k \in [K] \backslash \{k_t\}\)
\ENDFOR
\end{algorithmic}
\end{algorithm}

In Algorithm \ref{alg:ucb}, \(T_{k,t}\) is the number of times arm \(k\) has been played, up to (including) time \(t\), \(U_{k,s,t}\) is the upper confidence bound for arm \(k\) when it is played \(s\) times in total of \(t\) time steps, and \(B_{t,s}\) is the bonus term, where \(\beta\), \(\xi\) are constants defined in Assumption \ref{asp:process} and \(\alpha \in (0, \infty)\) is a tuning parameter that controls the exploration and exploitation trade-off. A tie is broken arbitrarily when selecting an arm in Line 4. 

Denote by \(\mu_{\max} = \max_{k \in [K]} \mu_k\) the optimal value with respect to the converged expectation, and by \(k_{\max} \in \argmax_{k \in [K]} \mu_k\) the corresponding optimal arm. We assume the optimal arm is unique.
Denote by \(\Delta_{\max} = \min_{k \in [K], k \ne k_{\max}} |\mu_{\max} - \mu_k|\) the gap between the optimal arm and the second optimal arm, and define
\begin{align}
    \label{eq:tau}
    \tau(t) = \bigl\lceil (\sfrac{2}{\Delta_{\max}} \cdot \beta^{\sfrac{1}{\xi}})^2 \cdot t^{\sfrac{2\alpha}{\xi}} \bigr\rceil.
\end{align}
Also, we denote by
\begin{align}
    \label{eq:tau_}
    \tau^\star = \min \bigl\{t \in \NN_+: t \ge \tau(t) \text{ and } 2R \tau(t) \ge \sqrt{t} + 2R(4K - 3)\bigr\} 
\end{align}
Denote by \(\bar X_n = \frac{1}{n} \sum_{k \in [K]} \sum_{t \in [T_{k,n}]} X_{k,t}\) the empirical mean reward under the Algorithm \ref{alg:ucb}.
The following theorem establishes theoretical analysis for Algorithm \ref{alg:ucb}.

\begin{theorem}[Theorem 3, \citet{shah2020non}]
    \label{thm:ucb}
    Under Assumptions \ref{asp:bounded} and \ref{asp:process}, the following holds for Algorithm \ref{alg:ucb} with any \(\alpha \in (2, \infty) \cap [\sfrac{\xi}{4}, \sfrac{\xi}{2})\).
    \begin{enumerate}
        \item Let \(\Delta_{\max} = \min_{k \in [K], k \ne k_{\max}} |\mu_{\max} - \mu_k|\). It then holds that
        \begin{align}
            \label{eq:ucb_convergence}
            \bigl|\EE[\bar{X}_{n}]-\mu_{\max}\bigr| & 
            \leq |\mu_{k_{\max},n} - \mu_{\max}| + \frac{2R (K-1) \cdot \Bigl(\bigl(\sfrac{2}{\Delta_{\max}} \cdot \beta^{\sfrac 1\xi}\bigr)^2 \cdot n^{\sfrac{2\alpha}{\xi}} + \frac{2}{\alpha - 2} + 1\Bigr)}{n}.
        \end{align}
        \item It holds for any $n \in \NN_+$ and $z \in [1, \infty)$ that
        \begin{align}
            \label{eq:ucb_concentration}
            \PP\big(n\bar{X}_{n} - n\mu_{\max} & \geq n^{\sfrac{2\alpha}{\xi}} z\big) \leq\frac{\beta'}{z^{\alpha -1}},
            \quad\PP\big(n\bar{X}_{n} - n\mu_{\max} \leq - n^{\sfrac{2\alpha}{\xi}} z\big) \leq\frac{\beta'}{z^{\alpha -1}},
        \end{align}
        where
        \begin{align*}
            \beta' = \max\Biggl\{2R(\tau^\star - 1)^{1 - \sfrac{2\alpha}{\xi}}, 2\bigl(2RK (\sfrac{2}{\Delta_{\max}})^{\sfrac{1}{\xi}}\bigr)^{\alpha-1} \cdot \max\biggl\{\beta, \frac{2(K - 1)}{(\alpha - 1) (1 + \tau(\tau^\star))^{\alpha - 1}}\biggr\}\Biggr\}.
        \end{align*}
    \end{enumerate}
\end{theorem}
Here, \(\Delta_{\max}\) is the gap between the optimal arm and the second optimal arm, \(\tau\) and \(\tau^\star\) are given by \eqref{eq:tau} and \eqref{eq:tau_}, respectively. Theorem \ref{thm:ucb} provides both the convergence guarantee \eqref{eq:ucb_convergence} and the concentration guarantee \eqref{eq:ucb_concentration} of the UCB algorithm described in Algorithm \ref{alg:ucb}. 

\paragraph{Min player counterpart.} We can construct a min player version of the non-stationary multi-arm bandit by adopting a negative bonus and choosing the arm with lowest UCB value. We include the algorithm here for completeness.

\begin{algorithm}[htbp]
\caption{A variant of Upper Confidence Bound (UCB) algorithm (Min Player)}
\begin{algorithmic}[1]
\label{alg:ucb_min}
\REQUIRE A non-stationary MAB with \(K\) arms. Parameters \(\alpha\), \(\beta\), \(\xi\).
\STATE Initialize \(T_{k,0} \leftarrow 0\) for any \(k \in [K]\).
\FOR{\(t \leftarrow 0, 1, \dots\)}
    \STATE \(U_{k,s,t} \leftarrow \sum_{\tau=1}^s X_{k,\tau} - B_{t,s}\), where \(B_{t,s} = \frac{\beta^{\sfrac{1}{\xi}} \cdot t^{\sfrac{\alpha}{\xi}}}{\sqrt{s}}\), for any \(s \in [t]\)
    \STATE \(k_t \leftarrow \argmin_{k \in [K]} U_{k, T_{k,t-1}, t-1}\) \quad{\color{blue}\texttt{/* Pull arm \(k_t\) and get reward \(X_{k_t, T_{k,t-1} + 1}\). */}}
    \STATE \(T_{k_t,t} \leftarrow T_{k_t,t-1} + 1\)
    \STATE \(T_{k,t} \leftarrow T_{k,t-1}, \forall k \in [K] \backslash \{k_t\}\)
\ENDFOR
\end{algorithmic}
\end{algorithm}

Similarly, we define \(\mu_{\min} = \min_{k \in [K]} \mu_k\), \(k_{\min} = \argmin_{k \in [K]} \mu_k\), and \(\Delta_{\min} = \min_{k \in [K], k \ne k_{\min}} |\mu_{\min} - \mu_k|\). Then, we mirror Theorem \ref{thm:ucb} to give the following analysis for the min player.

\begin{corollary}[UCB Min Player]
    \label{thm:ucb_min}
    Under Assumptions \ref{asp:bounded} and \ref{asp:process}, the following holds for Algorithm \ref{alg:ucb_min} with any \(\alpha \in (2, \infty) \cap [\sfrac{\xi}{4}, \sfrac{\xi}{2})\).
    \begin{enumerate}
        \item Let \(\Delta_{\min} = \min_{k \in [K], k \ne k_{\min}} |\mu_{\min} - \mu_k|\). It then holds that
        \begin{align}
            \label{eq:ucb_convergence_}
            \bigl|\EE[\bar{X}_{n}]-\mu_{\min}\bigr| & 
            \leq |\mu_{k_{\min},n} - \mu_{\min}| + \frac{2R (K-1) \cdot \Bigl(\bigl(\sfrac{2}{\Delta_{\min}} \cdot \beta^{\sfrac 1\xi}\bigr)^2 \cdot n^{\sfrac{2\alpha}{\xi}} + \frac{2}{\alpha - 2} + 1\Bigr)}{n}.
        \end{align}
        \item There exist constant $\beta' \in (1, \infty)$ depending on $R, K, \Delta_{\min}, \beta$,  $\xi$ and $\alpha$ such that it holds for any $n \in \NN_+$ and $z \in [1, \infty)$ that
        \begin{align}
            \label{eq:ucb_concentration_}
            \PP\big(n\bar{X}_{n} - n\mu_{\min} & \geq n^{\sfrac{2\alpha}{\xi}} z\big) \leq\frac{\beta'}{z^{\alpha -1}},
            \quad\PP\big(n\bar{X}_{n} - n\mu_{\min} \leq - n^{\sfrac{2\alpha}{\xi}} z\big) \leq\frac{\beta'}{z^{\alpha -1}}.
        \end{align}
    \end{enumerate}
\end{corollary}

\subsection{Leaf Level}
The leaf nodes at level \(H\) are children of nodes at level \(H - \half\) in the MCTS tree. Without loss of generality, we assume the leaf nodes correspond to the state where the max player is about to make a move, the same as the root node \(s_0\). Therefore, the nodes at level \(H - \half\) correspond to the min player. Consider node \(i \in [n_{H - \half}]\) at level \(H - \half\), corresponding to state \(s_i\). Each time the min player is at this node, it takes an action \(a_{H - \half} \in [K]\) and reaches the node \(s_H' = s^{(i)}_{H - \half} \circ a_{H - \half}\) at the leaf level \(H\). Then, the reward collected for the node \(s^{(i)}_{H - \half}\) and action \(a_{H - \half}\) is \(R(s^{(i)}_{H-\half} \circ a_{H - \half}) + \gamma \tilde v_H^\star(s_H')\), where \(\tilde v_H^\star(s_H')\) is calculated by the input value function proxy for Algorithm \ref{alg:mcts}. Thus, each time the node \(s^{(i)}_{H - \half}\) is visited, the reward is the summation of bounded independent and identical (for a given action) random variables and a deterministic evaluation. By Lemma \ref{lma:prelim}, there exists \(\beta_H\) such that the collected rewards at \(s^{(i)}_{H - \half}\) satisfy the concentration property \eqref{eq:prelim_concentration} and the convergence property \eqref{eq:prelim_convergence}.
This process mirrors the MAB problem we introduced in Section \ref{sec:mab}.
In order to apply Theorem \ref{thm:ucb} or Corollary \ref{thm:ucb_min}, we denote by \(R^{\max}_h\) the maximum magnitude of the rewards collected by all the nodes in level \(h \in \{0, \half, 1, \dots, H\}\) and we quantify it later in Lemma \ref{lma:magnitude}.
Then, we have the following lemma that establishes the convergence and concentration property of the nodes in level \(H - \half\).

\begin{lemma}[Leaf]
\label{lma:leaf}
Consider a node corresponding to state $s^{(i)}_{H-\half}$ at level $H-\half$ within the MCTS for $i \in [n_{H-\half}]$. Let $\tilde{v}^{(n)}_{H-\half}(s^{(i)}_{H-\half})$ be the total discounted reward collected at $s^{(i)}_{H-\half}$ during $n \in \NN_+$ visits of it, to one
of its $K$ leaf nodes under the UCB policy. Then, there exists appropriately large $\beta_{H-\half} \in (0, \infty)$ such that the following holds for a given $\xi_{H-\half} > 0$ and $\alpha_{H-\half} > 2$.
\begin{enumerate}
    \item It holds that
    \begin{align*}
        &\biggl|\EE\Bigl[\tilde v^{(n)}_{H-\half}(s^{(i)}_{H-\half})\Bigr] - n \cdot \tilde v_{H-\half}^\star(s^{(i)}_{H-\half}) \biggr|\\
        &\quad\leq 2 R^{\max}_{H-\half} (K-1)\cdot\Biggl(\biggl(
        \frac{2\beta_{H-\half}^{\sfrac{1}{\xi_{H-\half}}}}{\Delta_{H-\half}(s^{(i)}_{H-\half})} \biggr)^2\cdot n^{\frac{2\alpha_{H-\half}}{\xi_{H-\half}}}+\frac{2}{\alpha_{H-\half}-2} + 1\Biggr).
    \end{align*}
    \item There exist $\beta' \in (1, \infty)$ such that it holds for any $n \in \NN_+$ and $z \in [1, \infty)$ that
    \begin{align*}
        &\PP\Bigl(\tilde v^{(n)}_{H-\half}(s^{(i)}_{H-\half}) - n \cdot \tilde v_{H-\half}^\star(s^{(i)}_{H-\half})
        \geq n^{\frac{2\alpha_{H-\half}}{\xi_{H-\half}}} z\Bigr) \leq\frac{\beta'}{z^{\alpha_{H-\half} - 1}},\\
        & \PP\Big(\tilde{v}^{(n)}_{H-\half}(s^{(i)}_{H-\half}) - n \cdot \tilde v_{H-\half}^\star(s^{(i)}_{H-\half}) 
        \leq -n^{\frac{2\alpha_{H-\half}}{\xi_{H-\half}}} z\Big) \leq\frac{\beta'}{z^{\alpha_{H-\half} - 1}}.
    \end{align*}
\end{enumerate}
\end{lemma}
\begin{proof}
    The proof follows directly from Corollary \ref{thm:ucb_min}.
\end{proof}



\subsection{Recursion}
\label{sec:recursion}
Lemma \ref{lma:leaf} suggests that the convergence assumption \eqref{eq:ucb_convergence} and concentration assumption \eqref{eq:ucb_concentration} of Theorem \ref{thm:ucb} are satisfied by $\tilde{v}^{(n)}_{H-\half}$ for each node at level $H-\half$ with $\alpha^{H-\half}$ and $\xi^{(H-\half)}$ defined in Theorem \ref{thm:main} and with appropriately defined large enough constant $\beta_{H-\half}$.
We claim that result similar to Lemma \ref{lma:leaf}, but for node at level $H-1$, continues to hold with parameters $\alpha_{H-1}$ and $\xi_{H-1}$ as defined in Theorem \ref{thm:main} and with   appropriately defined large enough constant $\beta_{H-1}$. 
And similar argument will continue to apply going from level $h$ to $h-\half$ for all $h \in \{\half, 1, \dots, H - 1, H-\half\}$. That is, we shall assume that the convergence and concentration assumptions of Theorem \ref{thm:ucb} or Corollary \ref{thm:ucb_min} hold for $\tilde v_h(\cdot)$, for all nodes at level $h$ with parameters $\alpha^{(h)}$ and $\xi^{(h)}$ defined in Theorem \ref{thm:main} and with appropriately defined large enough constant $\beta_h$, and then argue that such holds for nodes at level $h-\half$ as well. Thus, we can prove the results for all $h \in \{\half, 1, \dots, H - 1, H-\half\}$ by induction. 

To that end, we first consider the nodes corresponding to the max player. 
For any \(h \in [H - 1]\), consider a node corresponding to state $s^{(i)}_h$ at level $h$ within the MCTS for $i \in [n_h]$.
%
As part of the algorithm, whenever this node is visited, one of the $K$ feasible action $a$ is taken and the node $s_{h + \half} = s^{(i)}_h \circ a$ at level $h + \half$ will be reached. This results 
in a reward $R(s^{(i)}_h, a) + \tilde v_h(s_{h + \half})$ at node $s^{(i)}_h$ at level $h$. Since $R(s, a)$ is an independent, bounded valued random variable while $\tilde v_h(\cdot)$ is effectively collected by following a path all the way to the leaf level. Inductively, we assume that $\tilde v_{h + \half}(\cdot)$ satisfies the convergence and concentration property for each node at level $h + \half$, with $\alpha_{h + \half}$ and $\xi_{h + \half}$ given by Theorem \ref{thm:main} and with appropriately defined large enough constant $\beta_{h + \half}$. Therefore, by an application of Lemma \ref{lma:prelim}, it follows that this combined reward continues to satisfy the convergence \eqref{eq:ucb_convergence} and concentration \eqref{eq:ucb_concentration} properties.
Thus, we invoke Theorem \ref{thm:ucb} and conclude the follow lemma.

\begin{lemma}
[Max player]
\label{lma:recursion}
Consider a node corresponding to state $s^{(i)}_h$ at level $h$ in the MCTS tree for $i \in [n_h]$. Let $\tilde v_h^{(n)}(s^{(i)}_h)$ be the sum discounted reward collected at $s^{(i)}_h$ during $n \in \NN_+$ visits. Then, the following holds for the choice of appropriately large $\beta_{h + \half} > 0$, for a given $\xi_{h+\half} > 0$ and $\alpha_{h+\half} > 2$.
\begin{enumerate}
    \item It holds that
    \begin{align*}
        \biggl|\frac{1}{n} \EE\Big[v^{(n)}_h(s^{(i)}_h)\Big] - \tilde v_h^\star(s^{(i)}_h) \biggr|
        &\leq \frac{2R^{\max}_h (K-1) \cdot \biggl(\Bigl(
        \frac{2\beta_h^{\sfrac{1}{\xi_h}}}{\Delta_h(s^{(i)}_h)} \Bigr)^2\cdot n^{\frac{2\alpha_h}{\xi_h}}+\frac{2}{\alpha_h-2} + 1\biggr)}{n}.
    \end{align*}
    \item There exist a large enough constant $\beta' \in (1, \infty)$ such that it holds for any $n \in \NN_+$ and $z \in [1, \infty)$ that
    \begin{align*}
        &\PP\Bigl(v^{(n)}_h(s^{(i)}_h) - n \cdot \tilde v_h^\star(s^{(i)}_h) \geq n^{\frac{2\alpha_h}{\xi_h}} z\Bigr) \leq\frac{\beta'}{z^{\alpha_h - 1}},\\
        &\PP\Bigl(v^{(n)}_h(s^{(i)}_h) - n \cdot \tilde v_h^\star(s^{(i)}_h) \leq -n^{\frac{2\alpha_h}{\xi_h}} z\Bigr) \leq\frac{\beta'}{z^{\alpha_h - 1}}.
    \end{align*}
\end{enumerate}
\end{lemma}
Next, we consider a node corresponding to state \(s^{(i)}_{h - \half}\) where the min player is about to make a move for any \(h \in [H-1]\).
Similarly, whenever this node is visited, one of the $K$ feasible action $a$ is taken and the node $s_h = s^{(i)}_{h-\half} \circ a$ at level $h$ will be reached. This results 
in a reward $R(s^{(i)}_{h-\half}, a) + \gamma \tilde v_{h-\half}(s_h)$ at node $s^{(i)}_{h-\half}$ at level $h-\half$. Here, the only difference is the \(\gamma\) factor, which does not effect our conclusion. By invoking Theorem \ref{thm:ucb}, we conclude the follow lemma.
\begin{lemma}
[Min player]
\label{lma:recursion_}
Consider a node corresponding to state $s^{(i)}_{h-\half}$ at level $h-\half$ in the MCTS tree for $i \in [n_{h-\half}]$. Let $\tilde v_{h-\half}^{(n)}(s^{(i)}_{h-\half})$ be the sum discounted reward collected at $s^{(i)}_{h-\half}$ during $n \in \NN_+$ visits. Then, the following holds for the choice of appropriately large $\beta_h > 0$, for a given $\xi_h > 0$ and $\alpha_h > 2$.
\begin{enumerate}
    \item It holds that
    \begin{align*}
        \biggl|\frac{1}{n}\EE\Bigl[\tilde{v}^{(n)}_{h-\half}(s^{(i)}_{h-\half})\Bigr] - \tilde v_{h-\half}^\star(s^{(i)}_{h-\half})\biggr|
        &\leq \frac{2R^{\max}_{h-\half} (K-1) \cdot \biggl(\Bigl(
        \frac{2\beta_{h-\half}^{\sfrac{1}{\xi_{h-\half}}}}{\Delta_{h-\half}(s^{(i)}_{h-\half})} \Bigr)^2 \cdot n^{\frac{2\alpha_{h-\half}}{\xi_{h-\half}}} + \frac{2}{\alpha_{h-\half}-2} + 1\biggr)}{n}.
    \end{align*}
    \item There exist a large enough constant $\beta' \in (1, \infty)$ such that it holds for any $n \in \NN_+$ and $z \in [1, \infty)$ that
    \begin{align*}
        &\PP\Bigl(\tilde{v}^{(n)}_{h-\half}(s^{(i)}_{h-\half}) - n \cdot \tilde v_{h-\half}^\star(s^{(i)}_{h-\half}) \geq n^{\frac{2\alpha_{h-\half}}{\xi_{h-\half}}} z\Bigr) \leq \frac{\beta'}{z^{\alpha_{h-\half} - 1}},\\
        &\PP\Bigl(\tilde{v}^{(n)}_{h-\half}(s^{(i)}_{h-\half}) - n \cdot \tilde v_{h-\half}^\star(s^{(i)}_{h-\half}) \leq -n^{\frac{2\alpha_{h-\half}}{\xi_{h-\half}}} z\Bigr) \leq \frac{\beta'}{z^{\alpha_{h-\half} - 1}}.
    \end{align*}
\end{enumerate}
\end{lemma}

\subsection{Error Analysis for Value Function Iteration}
We now move to the second part of the proof. The value function iteration improves the estimation of optimal value function by iterating Bellman equation. In effect, the MCTS tree is ``unrolling'' $2H$ steps of such an iteration. Specifically, let \(V(\cdot)\) denote the value function before an iteration, and let \(V'(\cdot)\) be the value function after iteration.
By definition, it holds for any $s \in \cS$, 
\begin{align}
    \label{eq:vi}
    V'(s) &= \max_{a \in [K]} \Bigl(\EE\bigl[R(s, a)\bigr] +  V(s \circ a)\Bigr).
\end{align}
Recall that value iteration is contractive with respect to $\|\cdot\|_\infty$ norm \citep{bertsekas2012dynamic}. That is, for any $h \geq 0$, 
\begin{align}
    \label{eq:contractive}
    \|V' - V^\star\|_\infty & \leq \gamma \cdot \|V - V^\star\|_\infty.
\end{align}
Since we can flipping the sign of \(R(\cdot)\), \(V(\cdot)\) and \(V'(\cdot)\) without breaking the contractive nature, the following iteration is also contractive with respect to \(\norm{\cdot}_\infty\) norm,
\begin{align}
    \label{eq:vi_}
    V'(s) &= \min_{a \in [K]} \Bigl(\EE\bigl[R(s, a)\bigr] + \gamma V(s \circ a)\Bigr).
\end{align}
Thus, \eqref{eq:contractive} still applies to the iteration corresponding to \eqref{eq:vi_}.
By drawing connection between \eqref{eq:value} and \eqref{eq:vi} as well as \eqref{eq:value_} and \eqref{eq:vi_}, we conclude the following lemma.
\begin{lemma}
\label{lma:vi}
The mean reward $\tilde v_0^\star(s_0)$ collected under the MCTS policy at root note $s_0$, starting with input value function proxy $V$ is such that 
\begin{align}
    \bigl|\tilde v_0^\star(s_0) - V^\star(s_0)\bigr| & \leq \gamma^H \|V - V^\star\|_\infty. 
\end{align}
\end{lemma}

\subsection{Completing Proof of Theorem~\ref{thm:main}}

In summary, using Lemma \ref{lma:leaf} and \ref{lma:recursion}, we conclude that the recursive relationship going from level $h-\half$ to $h-1$ and from level \(h\) to \(h-\half\) holds for all $h \in [H]$ with level $0$ being the root. 
We set \(\alpha_h = \frac{\alpha_{h+\half}}{4}\) and \(\xi_h = \alpha_{h+\half}\).
At root $s^{(0)}$, the query state that is input to the MCTS policy, we have that after $n$ total simulations of MCTS, the empirical average of the rewards over these $n$ trial, $\frac{1}{n} v^{(n)}_0(s_0)$ is such that using the fact that
\begin{align}
    \label{eq.final.1}
    \biggl|\frac{1}{n}\EE\bigl[v^{(n)}_0(s_0)\bigr] - \tilde v_0^\star(s_0)\biggr| 
    &= \cO\bigl(n^{\sfrac{2\alpha_0}{\xi_0} - 1}\bigr) 
    = \cO(\sfrac{1}{\sqrt n}), 
\end{align}
where the last equality follows from $\alpha_0 = \xi_0 / 4$.
By Lemma \ref{lma:vi}, it holds that
\begin{align}\label{eq.final.2}
    \bigl|\tilde v_0(s_0) - V^\star(s_0)\bigr| & \leq \gamma^H \varepsilon_0,
\end{align}
since $\varepsilon_0 = \|\hat{V} - V^\star\|_\infty$. Combining \eqref{eq.final.1} and \eqref{eq.final.2}, it then holds that
\begin{align}\label{eq.final.3}
    \biggl|\frac{1}{n}\EE\bigl[v^{(n)}_0(s_0)\bigr] - V^\star(s_0)\biggr| & \leq \gamma^H \varepsilon_0 + \cO(\sfrac{1}{\sqrt{n}}).
\end{align}
This concludes the proof of Theorem \ref{thm:main}.

\section{Pruning Analysis}

\subsection{Proof of the property of LSE}
\label{sec:proof_lse}
\begin{proof}
    [Proof of Lemma \ref{lma:lse}]
    Denote by $x^\star = \operatornamewithlimits{argmax}_{x\in\mathcal{X}}f(x)$. For any $\tau\ge0$, we have
\begin{align}
    &\lse(f,\mathcal{X},\tau) - \max_{x\in\cX}f(x)\notag\\
    &\quad= \frac{1}{\tau} \cdot\log \Bigl(\frac{1}{|\mathcal{X}|}\cdot \sum_{x\in\mathcal{X}} \exp(\tau \cdot f(x))\Bigr)-f(x^\star)\notag\\
    &\quad\ge \frac{1}{\tau}\cdot\log\Bigl(\frac{1}{|\mathcal{X}|}\cdot\exp(\tau\cdot f(x^\star))\Bigr)-f(x^\star)\notag\\
    &\quad= \frac{-\log|\mathcal{X}|}{\tau},\label{eq:1334}
\end{align}
where the first equality uses the definition of $x^\star$ and the first inequality uses the fact that $\exp(\tau\cdot f(x))\ge 0$ for any $x\in\cX$ and $\tau\ge0$.

Note that
\begin{align}
    &\max_{x\in\cX}f(x) - \lse(f,\cX,\tau)\\
    &\quad=\frac{1}{\tau} \cdot \log \Bigl(\frac{1}{|\mathcal{X}|}\cdot \sum_{x\in\mathcal{X}} \exp(\tau \cdot f(x))\Bigr)-f(x^\star)\notag\\
    &\quad\le \frac{1}{\tau}\cdot\log\Bigl(\exp(\tau\cdot f(x^\star))\Bigr)-f(x^\star)\notag\\
    &\quad=0\label{eq:1335}
\end{align}
where we use the definition of $x^\star$ and the fact that $\tau\ge0$.

Hence, we combine \eqref{eq:1334} and \eqref{eq:1335} to have
\begin{align}
    |\lse(f,\cX,\tau)-\max_{x\in\cX}f(x)|\le \frac{\log|\cX|}{\tau}, 
\end{align}
for any $\tau\ge0$. Thus, we finish the proof of Lemma \ref{lem:lse}. 
\end{proof}

\subsection{Proof of the pruning error bound}
\begin{proof}
    [Proof of Proposition \ref{prop:V_star_llm}]
     \label{pf:prop:V_star_llm}   By the notion of half-horizon, we mirror \eqref{eq:V_half_star}  to define
\begin{align}
    \tilde{V}_{\half}^\star(s) = \min_{a\in\tilde{\cA}} \Bigl(\EE\bigl[R(s, a)\bigr] + \gamma\cdot \Tilde{V}^\star(s \circ a)\Bigr).\label{eq:V_half_star_LLM}
\end{align}
Combining \eqref{eq:V_star_LLM}, \eqref{eq:V_half_star_LLM}, and the fact that $r(s, a, b) = R(s, a) + R(s \circ a, b)$, we have
\begin{align}
    \Tilde{V}^\star(s)= \max_{a \in \tilde{\cA}} \Bigl(\EE \bigl[R(s, a)\bigr]+ \Tilde{V}^\star_{\half}(s \circ a)\Bigr).\label{eq:V_half_star_LLM2}
    \end{align}

For the notational simplicity, we define $\hat \epsilon(\tau)$ as
\begin{align}
    \hat \epsilon(\tau) = \max_{j\in\{0,1/2\}} \max_{s\in\cS}\Bigl| \Bigl(\max_{a\in\cA}-\max_{a\in\tilde{\cA}}\Bigr)\bigl(Q_j^\star(s,a)\bigr)\Bigr|,\label{eq:eps_hat}
\end{align}
where we define $V_{0}^\star$ as $V^\star$, $\max_{x\in\mathcal{X}_1}f(x) -\max_{x\in\mathcal{X}_2}f(x)  $ as $(\max_{x\in\mathcal{X}_1}-\max_{x\in\mathcal{X}_2}) (f(x)) $, and $Q_j^{\star}(s,a)$ as
\begin{align}
    Q_j^{\star}(s,a)=\EE \bigl[(-1)^{2j}\cdot R(s, a)\bigr]+ (-\gamma)^{2j}\cdot V_{j}^\star(s \circ a)
\end{align}
for any $j\in\{0,1/2\}$. 

By the definition of $\hat \epsilon(\tau)$ in \eqref{eq:eps_hat}, we use triangle inequality to have
\begin{align}
    \hat \epsilon(\tau) &\le \max_{j\in\{0,1/2\}} \max_{s\in\cS}\Bigl| \max_{a\in\cA}Q_j^\star(s,a) - \lse(Q^\star_j(s,\cdot),\cA,\tau)\Bigr|\notag\\
    &\qquad +\max_{j\in\{0,1/2\}} \max_{s\in\cS}\Bigl| \lse(Q^\star_j(s,\cdot),\cA,\tau)-\lse(Q^\star_j(s,\cdot),\tilde{\cA},\tau)\Bigr|\notag\\&\qquad+ \max_{j\in\{0,1/2\}} \max_{s\in\cS}\Bigl| \max_{a\in\tilde{\cA}}Q_j^\star(s,a) - \lse(Q^\star_j(s,\cdot),\tilde{\cA},\tau)\Bigr|.\label{eq:Q_j2}
\end{align}
Invoking Lemma \ref{lem:lse}, we have
\begin{align}
    &\max_{j\in\{0,1/2\}} \max_{s\in\cS}\Bigl| \max_{a\in\cA}Q_j^\star(s,a) - \lse(Q^\star_j(s,\cdot),\cA,\tau)\Bigr| + \max_{j\in\{0,1/2\}} \max_{s\in\cS}\Bigl| \max_{a\in\tilde{\cA}}Q_j^\star(s,a) - \lse(Q^\star_j(s,\cdot),\tilde{\cA},\tau)\Bigr|\notag\\
    &\qquad\le 2\frac{\log(|\cA||\tilde{\cA}|)}{\tau}.\label{eq:Q_j3}
\end{align}
for any $\tau\ge 0$. 
Plugging \eqref{eq:Q_j2} and \eqref{eq:Q_j3} into \eqref{eq:eps_hat}, we bound $\hat{\epsilon}(\tau)$ as follows,
\begin{align}
    \hat{\epsilon}(\tau)&\le \frac{2\log(|\cA||\tilde{\cA}|)}{\tau}+\max_{j\in\{0,1/2\}} \max_{s\in\cS}\Bigl| \lse(Q^\star_j(s,\cdot),\cA,\tau)-\lse(Q^\star_j(s,\cdot),\tilde{\cA},\tau)\Bigr|\notag\\
    &\le \frac{2\log(|\cA||\tilde{\cA}|)}{\tau} +\epsilon(\tau)\label{eq:bound_eps_hat}
\end{align}
for any $\tau\ge0$. Here, the inequality uses the definition of $\epsilon(\tau)$ in Definition \ref{def:llm_approx}. 
By \eqref{eq:V_half_star2}, we have
\begin{align}
    |V^\star(s)-\Tilde{V}^\star(s)|&=\Bigl|\max_{a \in \cA} \Bigl(\EE \bigl[R(s, a)\bigr]+ V^\star_{\half}(s \circ a)\Bigr)\notag-\max_{a \in \tilde{\cA}} \Bigl(\EE \bigl[R(s, a)\bigr]
    + \Tilde{V}^\star_{\half}(s \circ a)\Bigr)\Bigr|\notag\\
    &\le \Bigl|\max_{a \in \tilde{\cA}} \Bigl(\EE \bigl[R(s, a)\bigr]+ V^\star_{\half}(s \circ a)\Bigr)\notag -\max_{a \in \tilde{\cA}} \Bigl(\EE \bigl[R(s, a)\bigr]
    + \Tilde{V}^\star_{\half}(s \circ a)\Bigr)\Bigr|+\hat {\epsilon}(\tau)\notag\\
    &\le \max_{a \in \tilde{\cA}} \Bigl|\EE \bigl[R(s, a)\bigr]+ V^\star_{\half}(s \circ a)\notag-\EE \bigl[R(s, a)\bigr]
    - \Tilde{V}^\star_{\half}(s \circ a)\Bigr| +\hat {\epsilon}(\tau)\notag\\
    &\le \max_{s\in\cS} \Bigl| V^\star_{\half}(s)
    - \Tilde{V}^\star_{\half}(s)\Bigr| +\hat {\epsilon}(\tau),\label{eq:v_star_eq_1}
\end{align}
for any $\tau\ge0$ and $s\in\cS$. Here, the first inequality uses the definition of $\hat {\epsilon}(\tau)$ in Definition \ref{def:llm_approx} and the triangle inequality, the second inequality uses the contraction property of $\max$ operator, and the last inequality uses the fact that $s\circ a\in\cS$. Take the maximum for $s\in\mathcal{S}$ on the left-hand side of \eqref{eq:v_star_eq_1}, we obtain
\begin{align}
    \max_{s\in\cS}|V^\star(s)-\Tilde{V}^\star(s)|\le \max_{s\in\cS} \Bigl| V^\star_{\half}(s)
    - \Tilde{V}^\star_{\half}(s)\Bigr| +\hat {\epsilon}(\tau),\label{eq:111}
\end{align}
for any $\tau\ge 0$. 
Emulating the similar proof, we have
\begin{align}
     |V_{\half}^\star(s)-\Tilde{V}_{\half}^\star(s)|&=\Bigl|\min_{a \in \cA} \Bigl(\EE \bigl[R(s, a)\bigr]+ \gamma\cdot V^\star(s \circ a)\Bigr)\notag-\min_{a \in \tilde{\cA}} \Bigl(\EE \bigl[R(s, a)\bigr]
    + \gamma\cdot \Tilde{V}^\star(s \circ a)\Bigr)\Bigr|\notag\\
    &= \Bigl|\max_{a \in \tilde{\cA}} \Bigl(\EE \bigl[-R(s, a)\bigr]- \gamma\cdot V^\star(s \circ a)\Bigr)\notag -\max_{a \in \tilde{\cA}} \Bigl(\EE \bigl[-R(s, a)\bigr]
    - \gamma\cdot\Tilde{V}^\star(s \circ a)\Bigr)\Bigr|\notag\\
    &\le \Bigl|\max_{a \in \tilde{\cA}} \Bigl(\EE \bigl[-R(s, a)\bigr]- \gamma\cdot V^\star(s \circ a)\Bigr)\notag -\max_{a \in \tilde{\cA}} \Bigl(\EE \bigl[-R(s, a)\bigr]
    - \gamma\cdot\Tilde{V}^\star(s \circ a)\Bigr)\Bigr|+\hat {\epsilon}(\tau)\notag\\
    &\le \max_{a \in \tilde{\cA}} \Bigl|\EE \bigl[-R(s, a)\bigr]- \gamma\cdot V^\star(s \circ a)\notag+\EE \bigl[R(s, a)\bigr]
    +\gamma\cdot \Tilde{V}^\star(s \circ a)\Bigr| +\hat {\epsilon}(\tau)\notag\\
    &\le \gamma\cdot\max_{s\in\cS} \Bigl| V^\star(s)
    - \Tilde{V}^\star(s)\Bigr| +\hat {\epsilon}(\tau),\label{eq:v_star_eq_2}
\end{align}
for any $\tau\ge0$ and $s\in\cS$. Here, the first inequality uses the definition of $\hat {\epsilon}(\tau)$ in Definition \ref{def:llm_approx} and the triangle inequality, the second inequality uses the contraction property of $\max$ operator, and the last inequality uses the fact that $s\circ a\in\cS$. Take the maximum for $s\in\mathcal{S}$ on the left-hand side of \eqref{eq:v_star_eq_2}, we obtain
\begin{align}
    \max_{s\in\cS}|V_{\half}^\star(s)-\Tilde{V}_{\half}^\star(s)|\le \gamma\cdot\max_{s\in\cS} \Bigl| V^\star_{\half}(s)
    - \Tilde{V}^\star_{\half}(s)\Bigr| +\hat {\epsilon}(\tau),\label{eq:112}
\end{align}
for any $\tau\ge 0$. Combining \eqref{eq:111} and \eqref{eq:112}, we have that 
\begin{align*}
    \max_{s\in\cS}|V^\star(s)-\Tilde{V}^\star(s)| &\le \frac{2}{1-\gamma}\cdot \hat{\epsilon}(\tau)\notag\\
    &\le  \frac{2}{1-\gamma}\cdot \Bigl({\epsilon}(\tau)+ \frac{2\log(|\cA||\tilde{\cA}|)}{\tau}\Bigr),
\end{align*}
for any $\tau\ge 0$. Here, the last inequality uses \eqref{eq:bound_eps_hat}. Thus, we finish the proof of Proposition \ref{prop:V_star_llm}. 
\end{proof}

\section{Auxiliary Lemmas}

\begin{lemma}[The Azuma-Hoeffding's Inequality \citep{azuma1967weighted}]
\label{lma:hoeffding}
Let \(X_1, \dots, X_n\) be independent random variables such that \(X_i \in [a_i, b_i]\) almost surely. It then holds for any \(t > 0\) that
\begin{align*}
    \PP\Biggl(\sum_{i=1}^n X_i - \EE\biggl[\sum_{i=1}^n X_i\biggr] \ge t\Biggr) &\le \exp\Biggl(-\frac{2t^2}{\sum_{i=1}^n (b_i - a_i)^2}\Biggr),\\
    \PP\Biggl(\sum_{i=1}^n X_i - \EE\biggl[\sum_{i=1}^n X_i\biggr] \le -t\Biggr) &\le \exp\Biggl(-\frac{2t^2}{\sum_{i=1}^n (b_i - a_i)^2}\Biggr),
\end{align*}
\end{lemma}

\begin{lemma}\label{lma:prelim}
Consider random variables $X_i, Y_i \in \RR$ for $i \in \NN_+$ such that $X_i$'s are independent and identically distributed taking values in $[-B, B]$ for some $B > 0$ and
$X_i$'s are independent of $Y_i$'s. Suppose there exists \(\mu_Y \in \RR\) such that it holds for any $n \in \NN_+$ that 
\begin{align*}
    \lim_{n \to \infty} \EE\Biggl[\frac1n \sum_{i=1}^n Y_i\Biggr] & = \mu_Y,
\end{align*}
and it holds for any $n \in \NN_+$ and $z \in [1, \infty)$ that
\begin{align*}
    &\PP\Biggl(\sum_{i=1}^n Y_i - n \mu_{Y} \geq n^{\eta} z\Biggr) \leq\frac{\beta}{z^{\xi}}, \quad
    \PP\Biggl(\sum_{i=1}^n Y_i - n \mu_{Y} \leq -n^{\eta} z\Biggr) \leq\frac{\beta}{z^{\xi}},
\end{align*}
where $\beta \in (1, \infty)$, $\xi \in (0, \infty)$, $\eta \in [\half, 1)$ are constants.
Let $Z_i = X_i + \rho Y_i$ for some $\rho > 0$, and let \(\mu_X = \EE[X_1]\) be the expectation of all \(X_i\)'s. Then,
    it holds for any $n \in \NN_+$ that
    \begin{align}
    \label{eq:prelim_concentration}
        \lim_{n\to\infty} \EE\Biggl[\frac1n \sum_{i=1}^n Z_i\Biggr] & = \mu_X + \rho \mu_Y. 
    \end{align}
    And it holds for any $n \in \NN_+$ and \(z \in [1, \infty)\) that
    \begin{align}
    \label{eq:prelim_convergence}
        &\PP\Biggl(\sum_{i=1}^n Z_i - n(\mu_{X} + \rho \mu_{Y}) \geq n^{\eta} z\Biggr) \leq\frac{\beta'}{z^{\xi}}, \quad 
        \PP\Biggl(\sum_{i=1}^n Z_i - n(\mu_{X} + \rho \mu_{Y}) \leq -n^{\eta} z\Biggr) \leq\frac{\beta'}{z^{\xi}},
    \end{align}
    where \(\beta' = (4\xi B^2 / e)^{\sfrac{\xi}{2}} + (2 \rho)^\xi \beta\).
\end{lemma}
\begin{proof}
Following from Lemma \ref{lma:hoeffding} and the fact that $X_i$'s are i.i.d. bounded random variables taking value in $[-B, B]$, it holds for any $t \geq 0$ that
\begin{align*}
    &\PP\Bigg(\sum_{i=1}^n X_i - n \mu_X \geq n t \Bigg) \leq \exp\Bigg(-\frac{t^2 n}{2B^2}\Bigg),\\
    &\PP\Bigg(\sum_{i=1}^n X_i - n \mu_X \leq - n t \Bigg) \leq \exp\Bigg(-\frac{t^2 n}{2B^2}\Bigg). 
\end{align*}
Thus, it holds for a large enough constant \(\beta'\) that
\begin{align*}
    \PP\Bigg(\sum_{i=1}^n Z_i - n(\mu_{X} + \rho \mu_{Y}) \geq n^{\eta} z\Bigg)
    & \leq \PP\Bigg(\sum_{i=1}^n X_i - n\mu_{X} \geq \frac{n^{\eta}z}{2} \Bigg) + 
    \PP\Bigg(\sum_{i=1}^n Y_i - n\mu_{Y} \geq \frac{n^{\eta}z}{2\rho} \Bigg) \nonumber \\
    & \leq \exp\biggl(-\frac{z^2 n^{2\eta -1}}{8B^2}\biggr) + \biggl(\frac{2 \rho}{z}\biggr)^\xi \beta.
\end{align*}
To choose a proper value value \(\beta'\), note that
\begin{align*}
    z^\xi \exp\biggl(-\frac{z^2 n^{2\eta -1}}{8B^2}\biggr)
    \le z^\xi \exp\biggl(-\frac{z^2}{8B^2}\biggr)
    \le (4\xi B^2)^{\sfrac{\xi}{2}} \exp(-\sfrac{\xi}{2}),
\end{align*}
where the first inequality follows from the fact that \(n \in [1, \infty)\) and \(\eta \in [\half, 1)\), and the second inequality is obtained via treating the right-hand side as a function of \(z\) and finding the maximum of that function.
Then, we conclude the proof of the first equation in \eqref{eq:prelim_convergence} by setting
\begin{align*}
    \beta' = (4\xi B^2 / e)^{\sfrac{\xi}{2}} + (2 \rho)^\xi \beta.
\end{align*}
We can prove the second equation by the same reasoning.
\end{proof}

\begin{lemma}
    [Maximum magnitude]
    \label{lma:magnitude}
    Denote by \(R^{\max}_h\) the maximum magnitude of the rewards collected by all the nodes in level \(h \in \{0, \half, 1, \dots, H\}\).
    It then holds for any \(h \in [H-1]\) that
    \begin{align}
        \label{eq:mag_leaf}
        &R^{\max}_{H-\half} = R^{\max} + \gamma(V^{\max} + \epsilon_0)\\
        \label{eq:mag_min}
        &R^{\max}_h = R^{\max} + R^{\max}_{h+\half}\\
        \label{eq:mag_max}
        &R^{\max}_{h-\half} = R^{\max} + \gamma R^{\max}_h.
    \end{align}
\end{lemma}
\begin{proof}
    Since all the rewards are bounded by \(R^{\max}\). Then it follows from the definition of \(V^\star\) and \(V^\star_{\half}\) in \eqref{eq:V_half_star} and \eqref{eq:V_half_star2} that
    \begin{align*}
        &V^\star(\cdot) 
        \le 2R^{\max} + \gamma \cdot \bigl(2R^{\max} + \gamma \cdot (2R^{\max} + \cdots)\bigr)
        = 2R^{\max} \cdot \Biggl(\sum_{i=0}^\infty \gamma^i\Biggr)
        \le \frac{2R^{\max}}{1 - \gamma}
        = V^{\max},\\
        &V^\star_{\half}(\cdot) 
        \le R^{\max} + \gamma \cdot \bigl(2R^{\max} + \gamma \cdot (2R^{\max} + \cdots)\bigr)
        \le 2R^{\max} \cdot \Biggl(\sum_{i=0}^\infty \gamma^i\Biggr)
        \le \frac{2R^{\max}}{1 - \gamma}
        = V^{\max}.
    \end{align*}
    The rewards collected for the node \(s\) and action \(a\) at level \(H - \half\) is \(R(s, a) + \gamma \tilde v_H^\star(s \circ a)\). Since the \(\tilde v_H^\star\) is provided by the input value function proxy \(V\) of Algorithm \ref{alg:ucb} and \(\norm{V - V^\star} \le \epsilon_0\), it holds that
    \begin{align*}
        \bigl|R(s, a) + \gamma \tilde v_H^\star(s \circ a)\bigr|
        &\le |R(s, a)| + \gamma \cdot \bigl(|V^{\max}(s \circ a)| + |\tilde v_H^\star(s \circ a) - V^{\max}(s \circ a)| + \bigr)\\
        &\le R^{\max} + \gamma (V^{\max} + \epsilon_0).
    \end{align*}
    For any \(h \in [H - 1]\), the rewards collected for the node \(s\) and action \(a\) at level \(h\) is
    \begin{align*}
        \bigl|R(s, a) + \tilde v_{h+\half}^\star(s \circ a)\bigr| 
        \le \bigl|R(s, a)\bigr| + \bigl|\tilde v_{h+\half}^\star(s \circ a)\bigr|
        \le R^{\max} + R^{\max}_{h+\half},
    \end{align*}
    and for level \(h - \half\),
    \begin{align*}
        \bigl|R(s, a) + \gamma \tilde v_h(s \circ a)\bigr| 
        \le \bigl|R(s, a)\bigr| + \gamma\bigl|\tilde v_h(s \circ a)\bigr|
        \le R^{\max} + \gamma R^{\max}_h,
    \end{align*}
    which concludes the proof of Lemma \ref{lma:magnitude}.
\end{proof}

\end{document}